\documentclass[letterpaper, 10 pt, conference]{ieeeconf}
\IEEEoverridecommandlockouts
\usepackage{cite}
\usepackage{amsmath,amssymb,amsfonts}
\usepackage[linesnumbered,ruled,vlined]{algorithm2e}
\usepackage{graphicx}
\usepackage{subcaption}
\usepackage{textcomp}
\usepackage{xcolor}
\usepackage{makecell}
\usepackage{adjustbox}
\usepackage{multirow}

\newtheorem{lemma}{Lemma}

\newtheorem{proposition}{Proposition}

\DeclareMathOperator*{\argmin}{arg\,min}
\begin{document}    

\title{\LARGE \bf Efficient Gradient-Based Inference for Manipulation Planning \\ in Contact Factor Graphs}

\author{Jeongmin Lee, Sunkyung Park, Minji Lee, and Dongjun Lee
\thanks{This research was supported by the National Research Foundation (NRF) grant funded by the Ministry of Science and ICT (MSIT) of Korea (RS-2022-00144468), and the Ministry of Trade, Industry \& Energy (MOTIE) of Korea (RS-2024-00419641). Corresponding author: Dongjun Lee.
}
\thanks{The authors are with the Department of Mechanical Engineering, IAMD and IOER, Seoul National University, Seoul, Republic of Korea.
\{ljmlgh,sunk1136,mingg8,djlee\}@snu.ac.kr. }
}

\maketitle

\begin{abstract}
This paper presents a framework designed to tackle a range of planning problems arise in manipulation, which typically involve complex geometric-physical reasoning related to contact and dynamic constraints. 
We introduce the Contact Factor Graph (CFG) to graphically model these diverse factors, enabling us to perform inference on the graphs to approximate the distribution and sample appropriate solutions. 
We propose a novel approach that can incorporate various phenomena of contact manipulation as differentiable factors, and develop an efficient inference algorithm for CFG that leverages this differentiability along with the conditional probabilities arising from the structured nature of contact.
Our results demonstrate the capability of our framework in generating viable samples and approximating posterior distributions for various manipulation scenarios.
\end{abstract}

\IEEEpeerreviewmaketitle

\section{Introduction}

Manipulation planning is a crucial and essential problem for robots to achieve their task goals. Given that typical manipulation tasks are executed through contact, the problem often arises various range of inverse problems over contact mechanics and structural dynamics, mostly requiring geometric-physical reasoning. 

This reasoning in manipulation planning typically involves considering a variety of factors, resulting in a non-convex, multi-modal search space.
One natural perspective for addressing these problems is to perform inference on the posterior distribution of the variables \cite{botvinick2012planning,levine2018reinforcement}. 
By framing the problem this way, one can apply diverse probabilistic inference methods, such as optimization, Markov Chain Monte Carlo (MCMC) or variational inference, to explore the posterior distribution and identify the most probable solutions. 
This approach has shown successful demonstrations in recent work on motion planning \cite{mukadam2018continuous,urain2022learning,yu2023gaussian}. 
However, in manipulation planning, the inherent complexity of contact interactions presents various challenges, such as creating a tractable formulation, developing scalable inference algorithms, and generating data using sampling methods. Consequently, finding an efficient way to seek a diverse set of solutions within the distribution remains elusive.

Given this context, in this work, we introduce a new modeling and inference algorithm to address diverse reasoning arise in manipulation planning, leveraging the differentiable representation and inherent structure of the contact mechanics. 
We first propose a formulation of the problem through a graphical model, termed Contact Factor Graphs (CFG). This model effectively organizes the relationships between object poses and contact forces, incorporating a variety of contact and dynamic factors. We describe how each factor can be formulated in \textit{differentiable} manner and demonstrate that various reasoning problems for manipulation can be represented within this framework.
Then, we introduce a novel gradient-based inference within the CFG. By capitalizing on the structured nature of the contact factors, we explain the advantages that can be obtained when using conditional probabilities in the inference process. We demonstrate the derivation of an efficient score function by exploiting the differentiability of the factors and the envelope theorem. 
The proposed framework is demonstrated to be capable of effective sample generation and approximation of the posterior probability distribution through diverse examples.

\section{Related Works} 

\subsection{Planning as Inference}

Formulating and solving problems using graphical models and inference provides a structured approach with a stochastic perspective. For instance, \cite{levine2018reinforcement} offers an overview of the inference perspective in control problems, while \cite{dellaert2021factor} elaborates on the utility of factor graphs in various robotic applications. Addressing the continuous motion planning challenge, \cite{mukadam2018continuous} advocates for leveraging this perspective through factor graph modeling, accompanied by maximum-a-posteriori (MAP) inference via optimization techniques. Uraín et al. \cite{urain2023composable} construct a reactive policy by sampling from a multimodal posterior distribution using the cross-entropy method.
The effectiveness of Bayesian inference in trajectory optimization is demonstrated in \cite{lambert2020stein,power2023constrained}.
In \cite{yang2023compositional}, geometric and physical constraint satisfaction problems are modeled as graphs and addressed through diffusion-based inference techniques. One advantage of such graphical models is their compositional nature. Our graphical model also employs contact-related constraints in a compositional manner but places greater emphasis on a model-based, differentiable energy function.

\subsection{Problems with Contact}

For problems involving contact, the inference approach introduces challenges as they involve constraints not only on the planning state but also on the contact impulse as nonlinear complementarity \cite{lee2022large}. 
This may necessitate searching over hybrid modes \cite{trinkle1991framework}, such as mixed-integer programming \cite{deits2019lvis}, which is typically computationally expensive. From an optimization perspective, contact-implicit constraints can make the problem mode-invariant, as demonstrated in \cite{mordatch2012contact,posa2014direct,howell2022calipso}. 
Recent works have focused more on improving the tractability and utility of the formulations. Some reflections on a ``good'' model for physical reasoning are discussed in \cite{toussaint2020describing}. 
This is also related to the recent emergence of differentiable simulators \cite{geilinger2020add,werling2021fast,howell2022dojo,lee2023differentiable}, which can provide gradients to help solve inverse problems involving contact physics. Several studies have demonstrated their efficacy compared to black-box sampling and inference methods commonly used in data generation processes and deep reinforcement learning \cite{turpin2022grasp,pang2023global,antonova2022rethinking}. 
Our work aligns with this trend, yet, instead of directly differentiating the simulation, we propose to form a factor graph with differentiable contact factors, which offers a broader perspective on modeling and facilitates tailored gradient-based inference.

\section{Modeling Contact Factor Graphs} \label{sec:modeling}

\subsection{Problem Formulation}

In this work, we primarily address the challenge of generating solution sets for geometric and physical constraints satisfaction encountered in contact-based tasks, assuming that the task skeleton is provided by high-level discrete search.
Also we base our analysis on given geometric and physical parameters, focusing on the variable $X = \left\{ q, u, \lambda \right\}$, where $q$, $u$, and $\lambda$ stand for the configuration space, external input, and contact force, respectively.
The constraints among these variables are modeled as factors, which can be visualized through a graphical model. In the following section, we detail the types of factors we primarily consider. 

\subsection{Contact Mechanics}
Incorporating contact constraints into the factor graph is essential and has significant implications. Here, we present a range of constraints involved in contact mechanics. 
One fundamental property of contact is its unilateral nature, resulting in the following constraints:
\begin{align} \label{eq:contact_unilateral}
\begin{split}
    &g(q) \ge 0 \quad \text{(non-penetration)} \\
    &g(q)\lambda = 0 \quad \text{(complementarity)}
\end{split}
\end{align}
where $g(q)\in\mathbb{R}$ is the gap function between objects.
Another crucial aspect of contact mechanics is the Coulomb friction law, which we precisely define as follows:
\begin{align} \label{eq:contact_coulomb}
\mathcal{C}_{n(q)} \ni
\mu\lambda_n + \lambda_t
\perp
\beta n(q) + v_t
\in \mathcal{C}_{n(q)} \quad \text{(Coulomb)}
\end{align}
Here, the subscript $n,t$ denote normal/tangential, $\mu \in \mathbb{R}^+$ is the friction coefficient, $\beta \in \mathbb{R}$ is an auxiliary variable, $n(q)$ is the contact normal, $\mathcal{C}_{n(q)}$ is a second-order cone with $n(q)$ as its axis.
The normal and tangential components of $\lambda$ are expressed as:
\begin{align*}
&\lambda_n = n(q)n(q)^T\lambda \quad \lambda_t = (I - n(q)n(q)^T)\lambda 
\end{align*}
Also the contact Jacobian $J_c$ and the tangential velocity $v_t$ are defined as\footnote{In this work, we express $\lambda$ and $J_c$ with respect to global coordinates. This obviates the need to explicitly specify the contact tangential coordinates, making it easier in computation and differentiation.}
\begin{align*}
&J_c = \begin{bmatrix}
    I & -[c(q)]_\times
\end{bmatrix} \\
&v_t = (I - n(q)n(q)^T)J_c\delta q 
\end{align*}
where $c(q)$ is the contact point and $[\cdot]_\times$ denotes the skew matrix. 
In the proposition below, we provide a rationale of the condition specified  \eqref{eq:contact_coulomb}:
\begin{proposition}
    The condition \eqref{eq:contact_coulomb} is equivalent to the Karush-Kuhn-Tucker conditions derived from the principle of maximal dissipation \cite{macklin2019nonsmooth}:
    \begin{align} \label{eq:coulomb_kkt}
    \begin{split}
        & 0 \le \| v_t \| \perp \mu\lambda_{n} - \| \lambda_{t} \| \ge 0 \\
        &\| v_t\|\lambda_{t} + \mu\lambda_{n} v_t = 0.
    \end{split}
    \end{align}
    while we slightly abuse notation here by denoting $\lambda_n$ as $n(q)^T\lambda$.
\end{proposition}
\begin{proof}
    Consider the case where $\mu\lambda_{n} - \| \lambda_{t} \| > 0$, which corresponds to the stick condition in \eqref{eq:coulomb_kkt}, and $\|v_t\|=0$ holds. 
    In \eqref{eq:contact_coulomb}, this case implies $\beta=0$ and consequently $\|v_t\|=0$, establishing equivalence.
    In the case of $\mu\lambda_{n} - \| \lambda_{t} \| = 0$, which aligns with the slip condition in \eqref{eq:coulomb_kkt}, the condition $\| v_t\|\lambda_{t} + \mu\lambda_{n} v_t = 0$ holds. In \eqref{eq:contact_coulomb}, this situation indicates that $\mu\lambda_n+\lambda_t$ resides on the boundary of $\mathcal{C}_{n(q)}$, and accordingly, $\beta n(q)+v_t$ must also lie on the boundary, where $\beta=\|v_t\|$. Consequently, $\lambda_t^Tv_t + \|\lambda_t\|\|v_t\|=0$, thus confirming their equivalence.
\end{proof}
Despite the equivalence of \eqref{eq:contact_coulomb} and \eqref{eq:coulomb_kkt}, we opt to utilize the proposed from \eqref{eq:contact_coulomb} in our framework to avoid explicit use of non-differentiable 2-norm $\| \cdot \|$ terms with a combination of conic constraints.

\subsection{Quasi-Dynamics}

To make planning results dynamically feasible, employing dynamics factor is also essential. 
In general, the relation can be formalized as follows:
\begin{align} \label{eq:factor_dyn}
A(q)\delta q - b(q) - J_u(q)^T u - J_c(q)^T\lambda = 0
\end{align} 
where $A(q)\in\mathbb{R}^{d\times d},~b(q)\in\mathbb{R}^d$ are the dynamics matrix/vector typically related to inertia and control gain, $J_u(q)$ is the input mapping matrix.
As we focus on manipulation problem in this work, we primarily adopt the quasi-static or quasi-dynamic assumption \cite{pang2023global} in \eqref{eq:factor_dyn}. 
While this may overlook certain dynamic effects, such as Coriolis forces, it generally leads to more well-defined inference process.

\subsection{Differentiable Contact Features} \label{subsec:geometry}

The aforementioned equations on contact constraints are composed of contact features such as gap $g(q)$, point $p(q)$, and normal $n(q)$.
As our objective is to develop an efficient \textit{gradient-based} inference scheme, it is essential to ensure the differentiability of factors.
To achieve this, we adopt a differentiable support function \cite{lee2023uncertain} as the geometry representation in our framework. This approach theoretically guarantee to provide differentiable contact features (DCF) and can represent arbitrary compositions of convex shapes. 
It is important to note that such differentiability of contact features between geometries is under-explored and has not been utilized enough even in recent literature \cite{pang2023global,le2024fast}, where shapes are often simplified to spheres or planes.

\subsection{Composition of Factors}

In manipulation planning, the graphical model should integrate the various constraint factors outlined earlier in a compositional manner. Often, we need to utilize a specific set of factors tailored to the task at hand, with each factor corresponding to a subset of variables. Below we present some examples of characterized sets of constraints for specific situations:
For static contact manipulation tasks like grasping or placement, the necessary factors may include:
\begin{align} \label{eq:set_static}
    &g(q) \ge 0
    \quad g(q)\lambda = 0
    \quad
\mu\lambda_n + \lambda_t
 \in \mathcal{C}_{n(q)}
\end{align}
with static equilibrium from \eqref{eq:factor_dyn}.
These functional constraints \eqref{eq:set_static} indicate the set of variables that satisfy contact complementarity and maintain stability with respect to rational contact forces.
Alternatively, if the objective is to maintain contact without slipping, as in pivoting scenarios, the following set of factors can be employed:
\begin{align} \label{eq:set_stick}
    g(q)=0 \quad v_t=0 \quad 
    \mu\lambda_n + \lambda_t
 \in \mathcal{C}_{n(q)}
\end{align}
with quasi-dynamics equation from \eqref{eq:factor_dyn}.
The above \eqref{eq:set_static} and \eqref{eq:set_stick} can be interpreted as distinct sets of contact factors. Similarly, various types of interactions can be modeled using a compositional arrangement of factors while the joint distribution on the variable nodes can be expressed as follows:
\begin{align} \label{eq:joint_distribution}
    p(X) = \prod_{i=1}^M p_i(X_i)~\text{where}~ p_i(X_i) \propto \text{exp}(-f_i)
\end{align}
where $M$ is the number of factor node, $X_i\subseteq X$ is the set of variable adjacent to $i$-th factor, $f_i$ is the energy function that encode the constraints. 
The form of the energy function depends on the type of constraints such as equality ($r_i=0$), inequality ($r_i\geq 0$), and cone constraints ($r_i\in\mathcal{C}$).
In this work, we assume the following forms of energy function $f_i$ for each type of constraint:
\begin{align} \label{eq:energy_function}
    \frac{1}{2}\| r_i \|^2, \quad 
    \frac{1}{2}\| \min(r_i,0) \|^2, \quad
    \frac{1}{2}\|d_\mathcal{C}(r_i) \|^2
\end{align}
which are for equality, inequality, and cone, respectively. Here, $d_\mathcal{C}$ is the L2 distance to the cone $\mathcal{C}$.
As mentioned earlier, each $r_i$ is readily differentiable. We also remark that our CFG formulation does not simply embed a forward simulator; instead, it provides the flexibility to model various types of interactions within physical reasoning.
Fig.~\ref{fig:overview} illustrates an example of CFG modeling.

\begin{figure*}[t]
  \centering
  \includegraphics[width=0.85\textwidth]{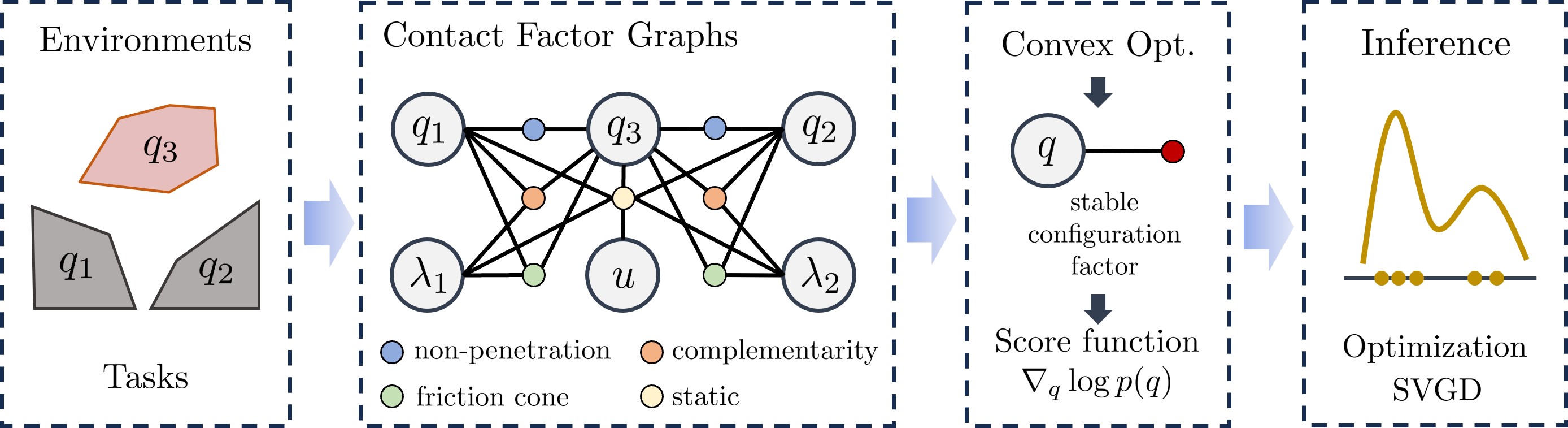} 
  \caption{Overview example of the proposed framework. Given the information on environment, contact factor graphs is composed using dynamics and contact factors relevant to the tasks. During the inference phase, convex optimization is employed to compress the distribution with respect to $q$. Score function is computed and applied within optimization or SVGD for inference.}
  \label{fig:overview}
\end{figure*}

\section{Inference in Contact Factor Graphs} \label{sec:solvingCFG}

Given the modeling described in Sec.~\ref{sec:modeling}, we need to perform inference on the joint distribution \eqref{eq:joint_distribution}.
Approximating such complex multimodal distributions presents a significant challenge in machine learning.
In this section, we propose an efficient Bayesian inference algorithm that exploits on the structure of the inner factorized likelihood function.

\subsection{MAP Inference on the Conditional Distribution}

Joint distribution on $X$ is typically high-dimensional, especially when the number of contact candidate increases. 
To alleiviate this issue, let us first consider the distribution on $u,\lambda$ conditioned by $q$.
Then we can perform MAP inference on the conditional distribution $p(u,\lambda \mid q)$ through following optimization: 
\begin{align} \label{eq:lower_convex}
    &\left( \lambda^*(q), u^*(q)\right) = \argmin_{\lambda,u} \sum_{i=1}^M f_i(\lambda,u,q). 
\end{align}
Our key insight is that for the factors \eqref{eq:contact_unilateral}, \eqref{eq:contact_coulomb} and \eqref{eq:factor_dyn}, and also their task-specific versions (e.g., \eqref{eq:set_static}), $u,\lambda$ exhibit linearity in $r_i$ given $q$. 
Then following lemma gives the convexity of the energy function form \eqref{eq:energy_function}:
\begin{lemma}
The energy functions defined in \eqref{eq:energy_function} are convex and have Lipschitz continuous gradient with respect to $r_i$.
\end{lemma}
\begin{proof}
For inequality and equality types, the properties are trivial. For cone types, we refer to Lemma 2 of \cite{castro2022unconstrained} for proof.
\end{proof}
Consequently, the optimization problem described in \eqref{eq:lower_convex} is convex, enabling us to perform MAP inference on the conditional distribution $p(u,\lambda \mid q)$ in an efficient manner, on a global scale.
As shown in Lemma 1, the gradient of this function possesses Lipschitz continuity, from which we can derive the Hessian matrix. For the Newton step, we can employ the semi-analytic primal (SAP) solver presented in \cite{castro2022unconstrained}, which utilizes exact linesearch and is proven to achieve superlinear convergence. 
Although this algorithm was originally developed for physics simulations, its theoretical properties are all preserved for \eqref{eq:energy_function}, allows it to be seamlessly adapted to our process.
We observe that this certifiable convex optimization provides significant convergence advantages in the inference process, as shown in Sec.~\ref{sec:experiment}. 
Additionally, our scheme can dramatically reduce the number of expensive contact feature computation function calls, which are usually numerical processes themselves.

\subsection{Conditional Independence and Factorization}

Aforementioned Newton steps for solving \eqref{eq:lower_convex} directly requires the factorization of the Hessian matrix.
Here, the structure of the matrix corresponds to the configuration of $p(u,\lambda \mid q)$, which we can exploit during the factorization process. 
For instance, independence arises between $u$ and $\lambda$ acting at different time intervals due to the Markov property of physics (i.e., $q$ at any point in time only affects $q, u, \lambda$ at the immediately preceding point in time).
Generally, the dependency between variables in conditional distribution is determined by the quasi-dynamics factors \eqref{eq:factor_dyn}, particularly through the structure of the Jacobians. In our framework, we effectively leverage this sparsity pattern during the factorization of the Hessian, which can be equivalently interpreted as an elimination process on the factor graph \cite{dellaert2021factor}.

\subsection{Inference on $q$} \label{subsec:q_inference}

Preceding MAP inference on the conditional distribution identifies the most plausible estimates $\lambda$ and $u$ for a given $q$. Building on this, we redefine the inference problem over the distribution of $q$ as follows:
\begin{align} \label{eq:q_distribution}
    p(q) \propto \prod_{i=1}^M \text{exp}(-f_i(\lambda^*(q),u^*(q),q)).
\end{align}
This distribution \eqref{eq:q_distribution} is not exactly a marginalization of the original distribution; rather, it represents the conditionally optimal distribution of $q$ from the MAP estimates $\lambda^*(q)$ and $u^*(q)$. 
Yet, this conditioned distribution remains multimodal and is intractable to compute directly. We should not merely assessing $q$ in isolation but considering how $q$ behaves when optimized estimates $\lambda^*(q)$ and $u^*(q)$ are factored into the evaluation.

\subsubsection{Score Function via Envelope Theorem}
To perform inference on the distribution given in \eqref{eq:q_distribution}, we can adopt various optimization, MCMC or variational inference algorithms. 
Here, the score function can play an essential role in the process from high-dimensional or constrained distributions. 
Within our framework, given that the MAP inference (e.g., \eqref{eq:lower_convex}) represents a parameterized convex optimization with respect to $q$, we propose to derive the score function for \eqref{eq:q_distribution} by employing the envelope theorem \cite{carter2001foundations}. 
Accordingly, we get
\begin{align} \label{eq:score_function}
    \nabla_q \log p(q) 
    = \sum_{i=1}^M \frac{\partial f_i(\lambda^*,u^*,q)}{\partial q}.
\end{align}
Note that the right hand side is composed of \textit{partial} derivative of $f_i$ with respect to $q$, therefore does not require any derivatives of $\lambda^*$ and $u^*$ with respect to $q$. 
It is worth noting that computing this derivative using implicit differentiation typically necessitates matrix inversion of problem-size matrices or a tailored design of the solver \cite{agrawal2019differentiating}, which can degrade the efficiency of the entire framework.
In the supplementary material, we summarize the results of specific derivation of partial derivatives for diverse contact factors.

\subsubsection{MAP Inference}

\begin{algorithm} [t] 
\SetAlgoLined
\caption{MAP Inference in CFG} 
\label{alg:cfg_mapinference}
\textbf{Input}: $\mathcal{G}$ \\
Initialize $l=0, q^0$, inverse Hessian $B^0$ \\
\While{not converged}{
Compute DCF for $q^l$ \\
Solve \eqref{eq:lower_convex} under $q^l$ using SAP solver \\
Compute score function $\nabla \log p(q^{l})$ \eqref{eq:score_function} \\
Compute $\Delta q^l = -B^l\nabla \log p(q^{l})$ \\
Compute $\alpha \leftarrow$ Line-Search($q^l,\Delta q^l$) \\
Update $q^{l+1}\leftarrow q^{l} + \alpha \Delta q^l$ \\
Update inverse Hessian $B^{l+1}$ via BFGS \\
$l\leftarrow l+1$
}
\textbf{Output:} $X$
\end{algorithm}

Following \eqref{eq:q_distribution}, conducting MAP inference in a CFG essentially entails solving the following optimization problem: 
\begin{align} \label{eq:map_inference}
    \min_q \sum_{i=1}^M f_i(\lambda^*(q),u^*(q),q)
\end{align}
Given that \eqref{eq:score_function} provides the gradient efficiently, using specific gradient-descent methods can facilitate the solution of \eqref{eq:map_inference}. 
Among the various algorithms available, we have chosen to use the Broyden–Fletcher–Goldfarb–Shanno (BFGS) algorithm \cite{fletcher2000practical}. This quasi-Newton method is particularly effective because it reliably converges by estimating the inverse Hessian matrix at each step, using only the gradient information provided. Alg.~\ref{alg:cfg_mapinference} summarizes the complete MAP inference process in CFG.

\subsubsection{Variational Inference}

\begin{algorithm} [t] 
\SetAlgoLined
\caption{Variational Inference in CFG} 
\label{alg:cfg_vinference}
\textbf{Input}: $\mathcal{G}$, kernel $\mathcal{K}$ \\
Initialize $l=0$, learning rate $\epsilon$, particle $\left\{ q^j \right\}_{j=1}^{S}$ \\
\While{not converged}
{
\For{$j=1:S$}
{
Compute DCF for $q^{j,l}$\\
Solve \eqref{eq:lower_convex} under $q^{j,l}$ using SAP solver \\
Compute kernel function $\mathcal{K}(q^{j,l},q)$ \\
Compute score function $\nabla \log p(q^{j,l})$ \eqref{eq:score_function} \\
Update particles \eqref{eq:svgd}
}
$l\leftarrow l+1$
}
\textbf{Output:} $X$
\end{algorithm}

The MAP inference in Alg.~\ref{alg:cfg_mapinference} typically yields only a single solution and thus cannot capture the multi-modality of the probabilistic distribution \eqref{eq:q_distribution}. 
Additionally, the solution obtained may be a local optimum and may not be sufficiently reasonable. Many problems involving contacts in robotics have multi-modal solutions, making it necessary to capture this variability.
In this context, variational inference on CFG can provide an overall approximation of the distribution, which can be formulated as
\begin{align} \label{eq:vi_inference}
    \hat{p}^*(q) = \argmin_{\hat{p}} D_{\text{KL}}(\hat{p}(q) \parallel p(q))
\end{align}
where $D_{\text{KL}}$ represents the Kullback-Leibler divergence. 

To address \eqref{eq:vi_inference}, we utilize Stein Variational Gradient Descent (SVGD \cite{liu2016stein}). In SVGD, the distribution $\hat{p}(q)$ is modeled non-parametrically using a set of particles. At each iteration, these particles are updated to minimize the divergence from the true distribution, utilizing techniques from reproducing kernel Hilbert space. The update process for each particle $q$ is as follows:
\begin{align} \label{eq:svgd}
\begin{split}
    &q^j \leftarrow q^j + \epsilon \eta(q^j) \\
    &\eta(q) = \sum_{j=1}^S \left[ \mathcal{K}(q^j,q) \nabla_{q^j} \log p(q^j) + \nabla_{q^j} \mathcal{K}(q^j,q) \right]
\end{split}
\end{align}
where $S$ is the number of particles, $\epsilon\in\mathbb{R}^+$ is a learning rate, $\mathcal{K}$ is a positive definite kernel function. 
Essentially, the process \eqref{eq:svgd} guides the particles to follow the score function, while a repulsive force prevents them collapse into a single mode.
As \eqref{eq:svgd} requires score function, this mechanism can be seamlessly integrated with the computation of \eqref{eq:score_function}.

Integrating with SVGD offers significant advantages, including the ability to perform operations in parallel across all particles and generally superior particle efficiency compared to MCMC-based sampling algorithms. The detailed variational inference process is presented in Alg.~\ref{alg:cfg_vinference}. Additionally, the overall framework is illustrated in Fig.~\ref{fig:overview}.

\section{Illustrative Examples} \label{sec:experiment}

In this section, we present various illustrative examples of manipulation planning using CFG.  
For the implementation, we primarily use MATLAB and C++, with \texttt{fminunc} and custom BFGS module. The planning results are validated through physics simulation \cite{lee2023modular}.

\subsection{Stable Placement}

\begin{figure}[t]
\centering
 \begin{subfigure}{0.2\textwidth}
    \includegraphics[width=\textwidth]{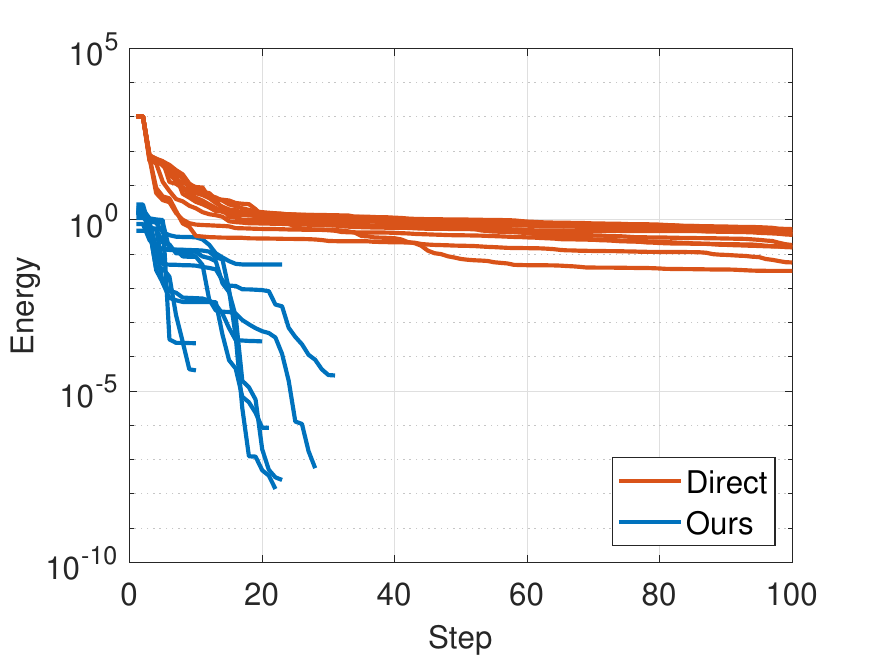}
  \end{subfigure}
  \begin{subfigure}{0.2\textwidth}
    \includegraphics[width=\textwidth]{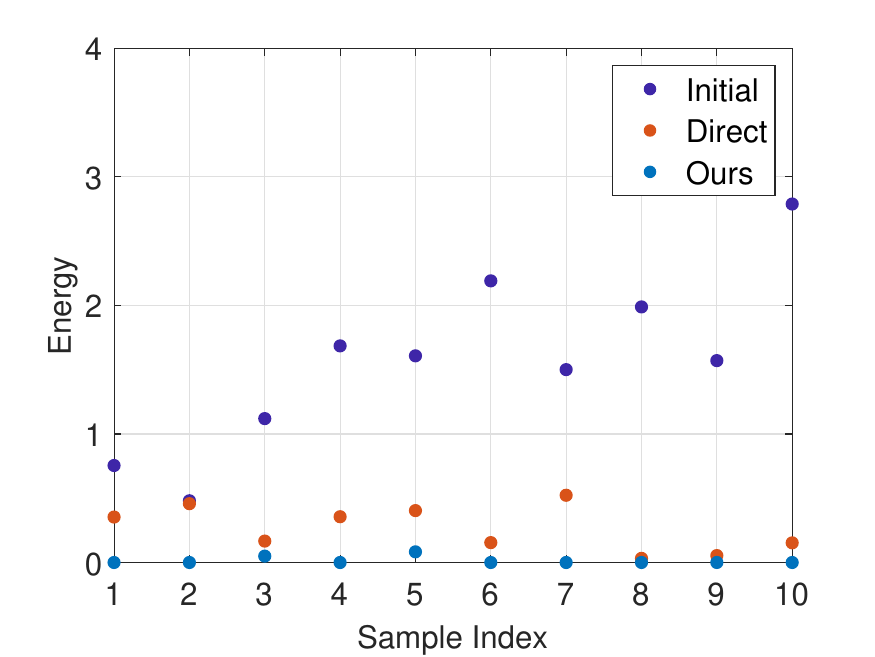}
  \end{subfigure}
  \caption{Comparison our MAP inference algorithm and direct inference on the joint distribution. Left: Convergence over iterations. Right: Quality of the final results.}
  \label{fig:stableplace_map}
\end{figure}

\begin{figure}[t]
\centering
 \begin{subfigure}{4cm}
    \includegraphics[width=\textwidth]{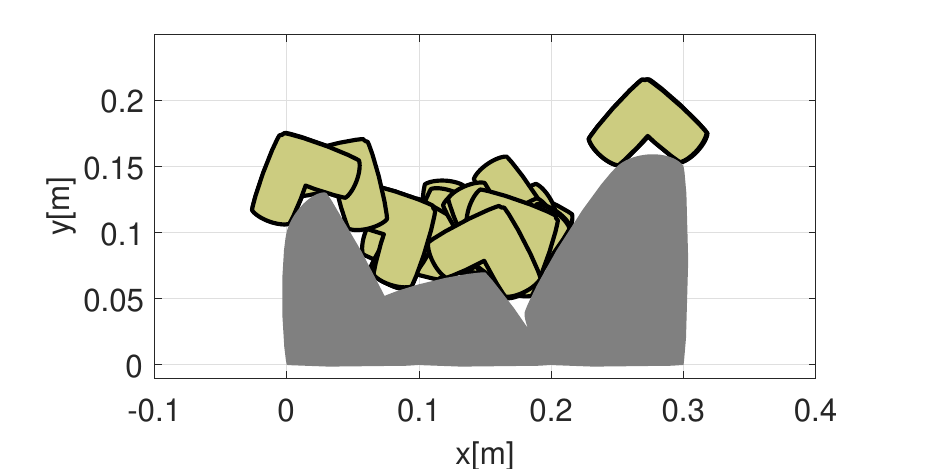}
  \end{subfigure}
  \begin{subfigure}{4cm}
    \includegraphics[width=\textwidth]{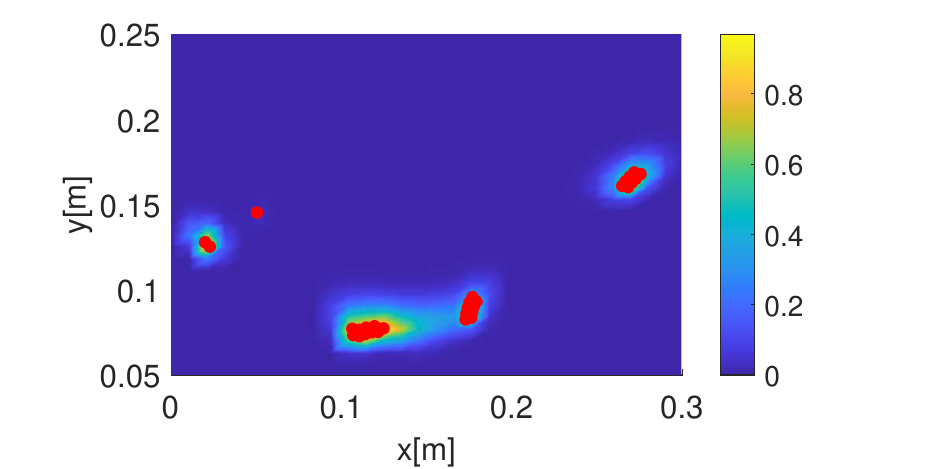}
  \end{subfigure}
  \caption{Examples of variational inference results in CFG. Left: Stable object poses sampled from the approximated distribution. Right: Visualization of the SVGD results under a fixed rotation.}
  \label{fig:stableplace_vi}
\end{figure}

In our first experimental scenario, we tackle the problem of inferring stable placement poses for a given object within a 2D environment. 
To address this problem, we apply the principles specified in \eqref{eq:set_static} within the CFG with equilibrium constraints under various external inputs $u$.
Note that we do not pre-specify contact pairs; instead, our CFG search navigates the space of making and breaking contacts.

We compare our inference scheme (Alg.~\ref{alg:cfg_mapinference}) with the baseline approach that performs inference directly on the joint distribution \eqref{eq:joint_distribution} by optimization (we refer to this as \texttt{direct} hereafter).
As shown in Fig.~\ref{fig:stableplace_map}, our method exhibits significantly faster convergence. For 10 randomly initialized samples, all of them successfully converge: 2 to (undesirable) local minima and 8 to proper stable poses. 
In the \texttt{direct} method, most of the samples fail to converge properly, and only $1$ of them succeed in finding a stable pose.

We also apply variational inference (Alg.~\ref{alg:cfg_vinference}) to approximate distribution of stable pose. 
For the kernel function, we use a common radial basis function.
The results depicted in Fig.~\ref{fig:stableplace_vi}, demonstrate that CFG modeling combined with gradient-based SVGD successfully generates particles that approximate the distribution, enabling us to sample diverse and high-quality poses.

\subsection{Pivoting}

\begin{figure}[t]
\centering
 \begin{subfigure}{0.15\textwidth}
    \includegraphics[width=\textwidth]{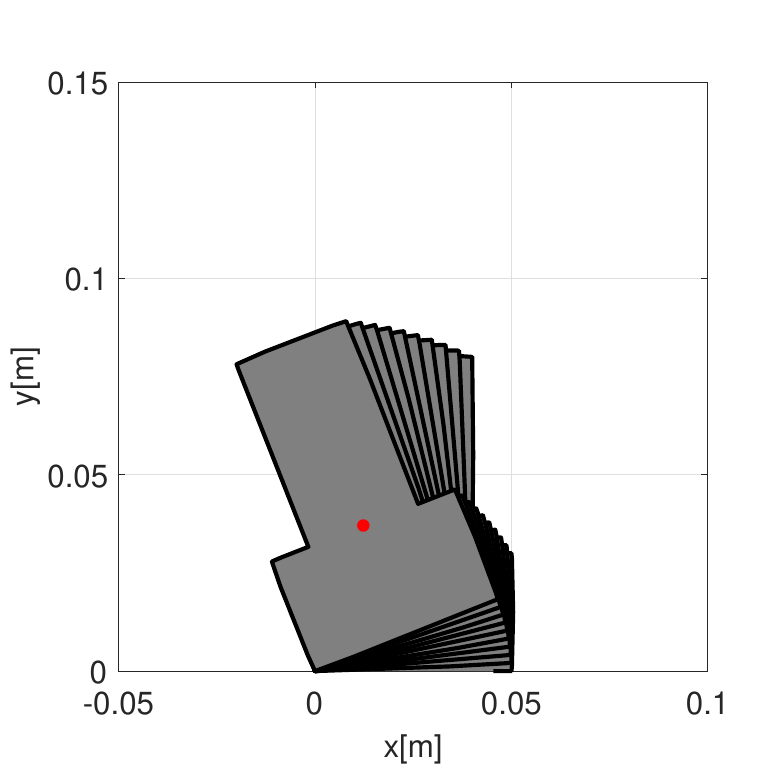}
  \end{subfigure}
  \begin{subfigure}{0.15\textwidth}
    \includegraphics[width=\textwidth]{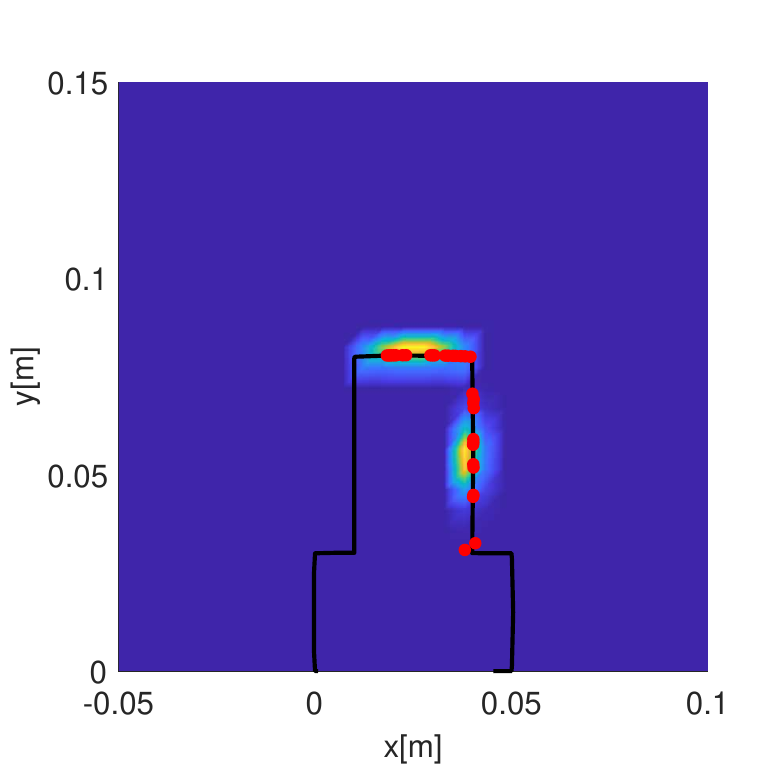}
  \end{subfigure}
  \begin{subfigure}{0.15\textwidth}
    \includegraphics[width=\textwidth]{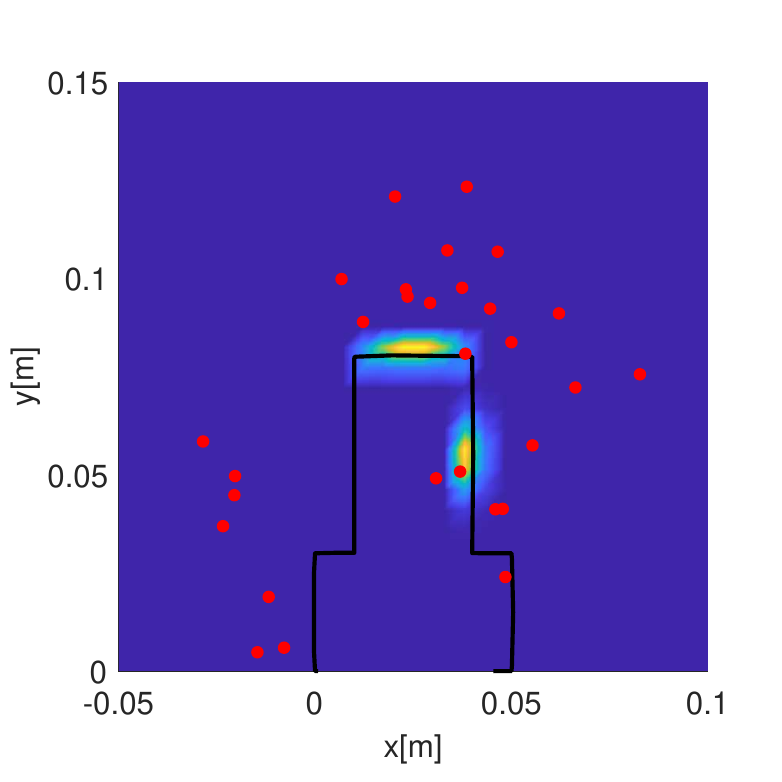}
  \end{subfigure}
  \caption{Left: Motion of pivoting manipulation. Middle: Visualization of SVGD results within our framework. Right: Visualization of ensemble MCMC results.}
  \label{fig:pivot_results}
\end{figure}

Next, we consider the pivoting task, which involves rotating an object while maintaining contact with both the floor and the manipulator. 
In CFG, we model the problem that reasons on the contact points (where the manipulator will push the object during the task) and the appropriate contact forces that achieve quasi-dynamics with the stick state \eqref{eq:set_stick} at each time step.  

Our results are depicted in Fig.~\ref{fig:pivot_results}. Given the clear multimodality of the possible pushing point, we test variational inference in CFG. 
As shown, the resulting particles from Alg.~\ref{alg:cfg_vinference} successfully capture the distribution. We also compared these results with ensemble MCMC \cite{foreman2013emcee}, a widely-used Bayesian inference method. Although the ensemble scheme offers advantages in exploration, the resulting samples fail to properly capture the distribution, possibly due to the complex contact landscape and non-reliance on gradient information.

\subsection{Nonprehensile Valve Turning}

\begin{figure}[t]
\centering
    \includegraphics[width=8cm]{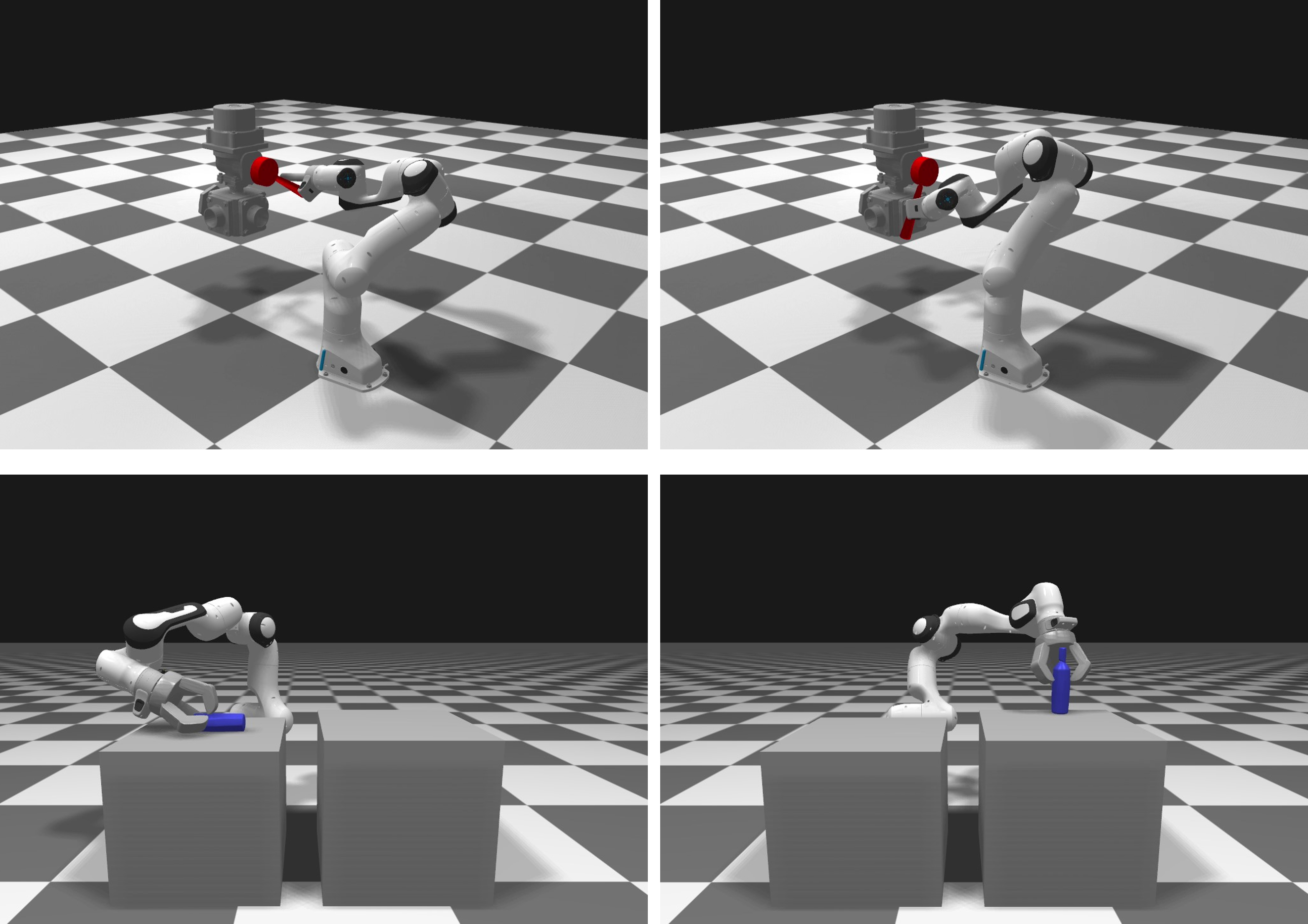}
    \caption{Snapshots of planning results generated through inferences in CFG. Top: valve turning with slide maneuvers. Bottom: multifinger grasping and placing.}
    \label{fig:3d_snapshots}
\end{figure}

Subsequently, we test a scenario in which the robot uses its end effector to turn a lever on a valve. Here, we give the critical constraint that the robot must maintain contact with the lever and turn it while avoiding reaching its joint limit. Therefore, it should effectively exploit sliding on the contact surface.
As illustrated in the compressed graph in Fig.~\ref{fig:cfg_seq}, we model the planning problem as a sequence of touch and turn actions. The touch factor is represented by $g(q) = 0$, while the turn factor incorporates $g(q) = 0$, the desired lever angle, joint limits, and the friction law \eqref{eq:contact_coulomb}.

Then as in integrated task and motion planning \cite{garrett2021integrated}, we can sequentially apply the CFG inference process in a factored manner to derive the complete solution. Initially, we sample the touching pose $q_0$, followed by the turning trajectory $q_1, \cdots, q_N$. 
Across $10$ different lever geometries, we successfully generate planning results in an average of $74~\rm{ms}$ (see Fig.~\ref{fig:3d_snapshots} for a snapshot) using Alg.~\ref{alg:cfg_mapinference}, whereas the \texttt{Direct} method failed to produce any successful samples.

\subsection{Multifinger Grasp and Place}

\begin{figure}[t]
\centering
\begin{subfigure}{0.225\textwidth}
    \includegraphics[width=\textwidth]{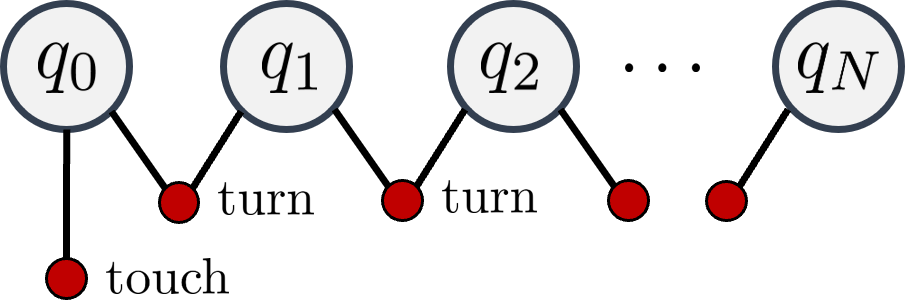}
    \caption{valve turn}
\end{subfigure}   
\begin{subfigure}{0.15\textwidth}
    \includegraphics[width=\textwidth]{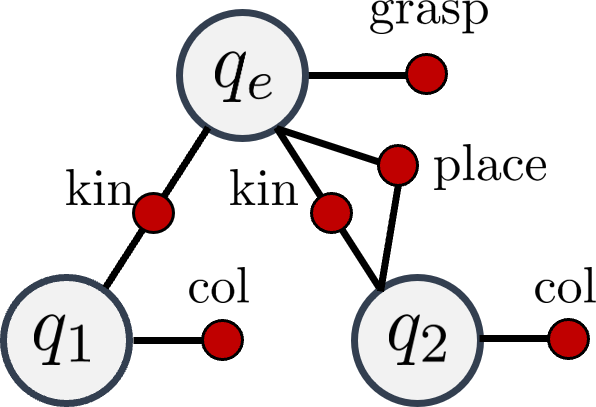}
    \caption{grasp place}
\end{subfigure}
\caption{Compressed form of CFG with respect to $q$ (as in Fig.~\ref{fig:overview}) for manipulation planning problems. For better efficiency, inference can be performed sequentially on each subgraph.}
\label{fig:cfg_seq}
\end{figure}

\begin{figure}[t]
\centering
 \begin{subfigure}{0.2\textwidth}
    \includegraphics[width=\textwidth]{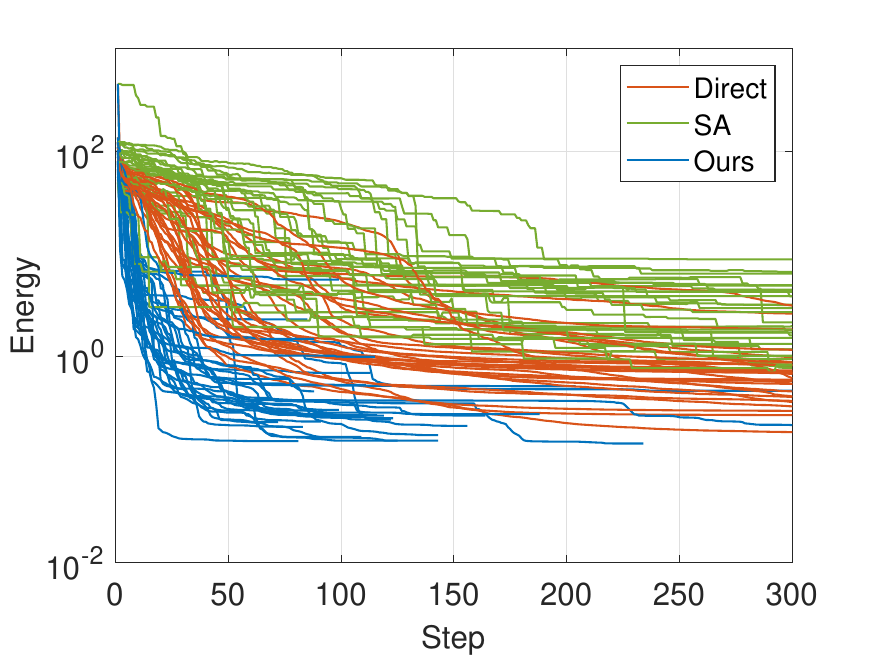}
  \end{subfigure}
  \begin{subfigure}{0.2\textwidth}
    \includegraphics[width=\textwidth]{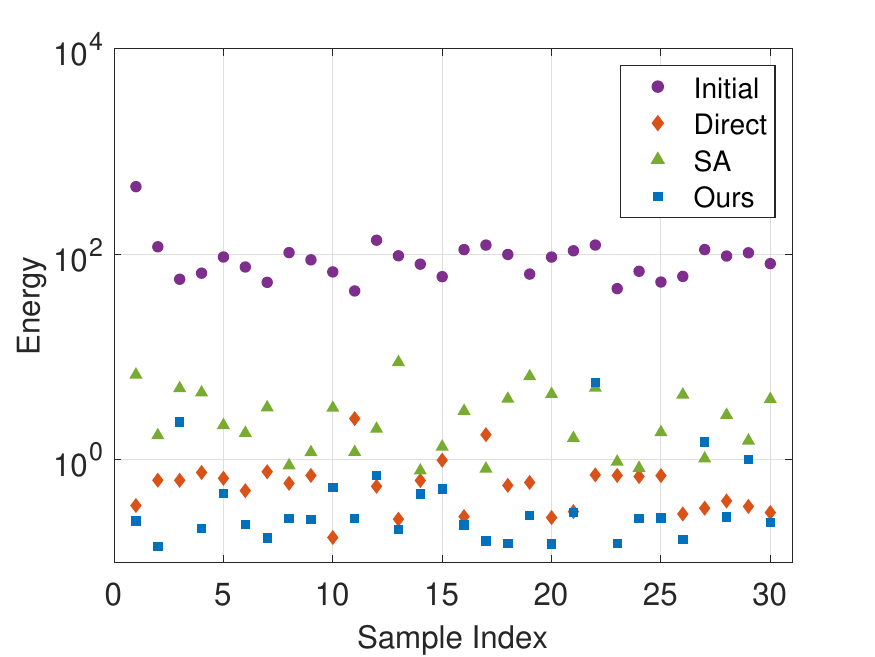}
  \end{subfigure}
  \caption{Comparison our MAP inference algorithm and direct inference on the joint distribution/simulation annealing. Left: Convergence over iterations. Right: Quality of the final results.}
  \label{fig:grasp_result}
\end{figure}

Finally, we tackle the multifinger grasp and place problem, which involves inferring proper grasping poses for a gripper and a given object, and then place it in a stable pose. 
Similar to the valve turning example, we can represent and solve the problem sequentially: first, find the gripper pose $q_e$ with manipulator pose $q_1$, then determine the placement pose $q_2$ as depicted in Fig.~\ref{fig:cfg_seq}. We include kinematic feasibility, collision avoidance, and the set of contact factors \eqref{eq:set_static} to naturally handle complementarity and stability, while eliminating the need to pre-specify the contact interface.

By aforementioned strategy, we successfully generate planning results in an average of $384~\rm{ms}$ (see Fig.~\ref{fig:3d_snapshots} for a snapshot).
To validate the sample generation performance, we compare $(q_e,q_1)$ sampling using Alg.~\ref{alg:cfg_mapinference}, with the \texttt{direct} method and also simulated annealing (SA), which performs global optimization through sampling without relying on gradient information. 
The results, depicted in Fig.~\ref{fig:grasp_result}, demonstrate that our methods achieve significantly faster convergence and yield higher-quality solutions. As other methods exhibit very low sampling success rates, the overall planning time slowed down by more than $10\times$.

\section{Discussion and Conclusion}

In this paper, we introduce the Contact Factor Graph (CFG) framework, which serves as a versatile tool for addressing a wide geometric-physical reasoning problems arise in manipulation planning. 
CFG facilitates the reasoning by modeling a differentiable, factorized probabilistic distribution aligned with contact mechanics and dynamics. 
We further present an inference algorithm for CFG, employing parameterized convex optimization techniques that leverage efficient gradient computation through the envelope theorem.

One limitation of the current framework is our assumption that the skeleton of manipulation planning is predefined.
As partially demonstrated in Sec.~\ref{sec:experiment}, we believe that various reasoning in CFG can be extended by sequential sampling on subgraphs, therefore can be combined to a range of works on task-level planning \cite{toussaint2015logic,migimatsu2020object,garrett2021integrated} or subgoal sampling \cite{pang2023global}, to significantly improve the tractability of the diverse array of complex manipulation. 
Also, our efficient and scalable solution sample generation scheme can also be used for data generation to build learning-based models, particularly for learning implicit distributions for diffusion processes \cite{yang2023compositional}.

\newpage
\bibliographystyle{unsrt}
\bibliography{reference}

\end{document}